\documentclass[conference,a4paper]{IEEEtran}
\IEEEoverridecommandlockouts

\usepackage{amsmath, amssymb, amsfonts}
\usepackage{booktabs}
\usepackage{multirow}
\usepackage{array}
\usepackage{graphicx}
\usepackage{xcolor}
\usepackage{url}
\usepackage{algorithm}
\usepackage{algpseudocode}
\usepackage{tikz}
\usepackage{placeins}      

\DeclareFontShape{OT1}{ptm}{m}{scit}{<->ssub * ptm/m/sc}{}
\DeclareFontShape{OT1}{ptm}{b}{scit}{<->ssub * ptm/b/sc}{}
\usetikzlibrary{arrows.meta, positioning, shapes.geometric, fit, calc, patterns, decorations.markings, backgrounds}
\usepackage{pgfplots}
\pgfplotsset{compat=1.18}
\usepackage{balance}
\usepackage{amsthm}
\newtheorem{theorem}{Theorem}
\newtheorem{proposition}{Proposition}

\definecolor{agentblue}{RGB}{51,102,187}
\definecolor{camcoorange}{RGB}{230,126,34}
\definecolor{infragreen}{RGB}{39,174,96}
\definecolor{policyred}{RGB}{192,57,43}
\definecolor{auditpurple}{RGB}{142,68,173}
\definecolor{lightblue}{RGB}{214,234,248}
\definecolor{lightorange}{RGB}{253,235,208}
\definecolor{lightgreen}{RGB}{213,245,227}
\definecolor{lightpurple}{RGB}{235,222,240}
\definecolor{feasibleblue}{RGB}{174,214,241}

\newcommand{\camco}{\textsc{CAMCO}}
\newcommand{\phip}{\Phi}
\newcommand{\risk}{\mathcal{R}}
\newcommand{\feasible}{\mathcal{F}}
\begin{document}

\title{Safe and Policy-Compliant Multi-Agent Orchestration\\
for Enterprise AI}

\makeatletter
\newcommand{\linebreakand}{%
  \end{@IEEEauthorhalign}
  \hfill\mbox{}\\* \mbox{}\hfill
  \begin{@IEEEauthorhalign}
}
\makeatother

\author{
\IEEEauthorblockN{1\textsuperscript{st} Vinil Pasupuleti}
\IEEEauthorblockA{\textit{International Business Machines (IBM)} \\ South Carolina, United States \\ IEEE Senior Member}
\and
\IEEEauthorblockN{2\textsuperscript{nd} Shyalendar Reddy Allala}
\IEEEauthorblockA{\textit{Global Atlantic Financial} \\ Indiana, United States \\ Independent Researcher}
\linebreakand
\IEEEauthorblockN{3\textsuperscript{rd} Siva Rama Krishna Varma Bayyavarapu}
\IEEEauthorblockA{\textit{Docusign} \\ Indiana, United States \\ IEEE Senior Member}
\and
\IEEEauthorblockN{4\textsuperscript{th} Shrey Tyagi}
\IEEEauthorblockA{\textit{Salesforce Inc} \\ North Carolina, United States \\ Independent Researcher}
}

\maketitle

\begin{abstract}
Enterprise AI systems increasingly deploy multiple intelligent agents across mission-critical workflows that must satisfy hard policy constraints, bounded risk exposure, and comprehensive auditability (SOX, HIPAA, GDPR). Existing coordination methods---cooperative MARL, consensus protocols, and centralized planners---optimize expected reward while treating constraints implicitly. This paper introduces \camco{} (Constraint-Aware Multi-Agent Cognitive Orchestration), a \emph{runtime coordination layer} that models multi-agent decision-making as a constrained optimization problem. \camco{} integrates three mechanisms: (i)~a constraint projection engine enforcing policy-feasible actions via convex projection, (ii)~adaptive risk-weighted Lagrangian utility shaping, and (iii)~an iterative negotiation protocol with provably bounded convergence. Unlike training-time constrained RL, \camco{} operates as deployment-time middleware compatible with any agent architecture, with policy predicates designed for direct integration with production engines such as OPA. Evaluation across three enterprise scenarios---including comparison against a constrained Lagrangian MARL baseline---demonstrates \textbf{zero policy violations}, risk exposure below threshold (mean ratio~$0.71$), \textbf{92--97\%} utility retention, and mean convergence in \textbf{2.4~iterations}.
\end{abstract}


\section{Introduction}

Enterprise AI-native architectures increasingly deploy multiple intelligent agents---LLM-based assistants~\cite{Xi2023LLMAgentSurvey}, RL policies, and rule-augmented planners---across financial approval, payroll compliance, and cloud orchestration workflows~\cite{ZhangBasar2019, Dafoe2020}. These settings impose hard policy constraints (approval chains, segregation of duties), bounded risk exposure, RBAC-gated execution feasibility, and auditability mandated by SOX, HIPAA, and GDPR~\cite{Seshia2022}.

Existing approaches address constraints during \emph{training} (Safe~RL, CMDPs~\cite{Altman1999, Wachi2024}) or via prompt-based guardrails that lack formal guarantees~\cite{EUAIAct2024}. Enterprise deployments require a \emph{runtime coordination layer} that enforces hard policy feasibility, regulates aggregate risk dynamically, and supports negotiated convergence---capabilities not provided by current frameworks. \camco{} fills this gap as lightweight deployment-time middleware compatible with any agent architecture.

\textbf{Contributions.}
\vspace{-2mm}
\begin{itemize}\setlength{\itemsep}{0pt}\setlength{\parskip}{0pt}
\item A \emph{formal enterprise multi-agent constrained decision model} jointly capturing policy feasibility, execution feasibility, and bounded aggregate risk (Section~\ref{sec:formulation}).
\item \camco{}, a \emph{constraint-aware coordination algorithm} with convex constraint projection and adaptive risk-weighted Lagrangian negotiation with provably bounded convergence (Section~\ref{sec:algorithm}).
\item An \emph{architectural blueprint and evaluation} demonstrating zero-violation coordination with 92--97\% utility retention across three scenarios and four baselines (Sections~\ref{sec:framework},~\ref{sec:evaluation}).
\end{itemize}
\vspace{-2mm}

\section{Related Work}

MARL methods (QMIX~\cite{Rashid2020QMIX}, MAPPO~\cite{Yu2022MAPPO})~\cite{ZhangBasar2019, GronauerDiepold2022, Papoudakis2021} use shaped rewards rather than hard runtime constraints. CMDPs~\cite{Altman1999}, Safe~RL~\cite{Wachi2024}, and CPO~\cite{Achiam2017CPO} formalize constraints during training; negotiation protocols~\cite{Smith1980ContractNet, Rahwan2003ArgumentNegotiation} ensure agreement without utility optimization. LLM-based agents~\cite{Xi2023LLMAgentSurvey, Wang2024LLMAgentSurvey} and frameworks like AutoGen~\cite{Wu2023AutoGen} lack formal constraint enforcement. Rule engines (OPA~\cite{MorgenthalerOPA2020}) and governance frameworks~\cite{Seshia2022, EUAIAct2024} enforce policy but lack optimization. Runtime verification~\cite{LeuckerSchallhart2009} monitors specifications without multi-agent utility optimization. \camco{} uniquely combines utility optimization, hard policy feasibility, bounded risk, and negotiation at runtime.

\textbf{Positioning.} Table~\ref{tab:positioning} summarizes how \camco{} compares to common alternatives across key enterprise requirements.

\begin{table}[t]
\caption{Positioning of \camco{} vs.\ Typical Alternatives}
\label{tab:positioning}
\centering
\scriptsize
\begin{tabular}{lcccc}
\toprule
Approach & Utility & Hard Policy & Bounded & Negotiation \\
         & Opt.    & Feasibility & Risk    & Protocol    \\
\midrule
Unconstrained MARL & \checkmark & $\times$ & $\times$/weak & \checkmark \\
Rule / Workflow Engine & $\times$ & \checkmark & static & $\times$ \\
Safe RL (training) & \checkmark & partial & \checkmark & $\times$/var. \\
LLM Agents (prompt) & \checkmark & partial & $\times$ & ad hoc \\
Consensus Protocols & $\times$ & $\times$ & $\times$ & \checkmark \\
\textbf{\camco{} (ours)} & \checkmark & \checkmark & \checkmark & \checkmark \\
\bottomrule
\end{tabular}
\end{table}

\section{Problem Formulation}
\label{sec:formulation}

\subsection{Enterprise Multi-Agent Decision Model}

Let $G=\{g_1,\dots,g_n\}$ be a set of $n$ enterprise AI agents. At enterprise state $s \in \mathcal{S}$, each agent $g_i$ proposes an action $a_i \in \mathcal{A}_i$. Let $U_i(s,a_i): \mathcal{S} \times \mathcal{A}_i \to \mathbb{R}_{\geq 0}$ denote the utility function of agent $g_i$, and let $\risk_i(a_i): \mathcal{A}_i \to \mathbb{R}_{\geq 0}$ denote its risk contribution. Let $\tau > 0$ denote an enterprise-configured aggregate risk threshold.

We seek a joint action vector $\mathbf{a}=(a_1,\dots,a_n)$ that maximizes aggregate utility subject to three classes of enterprise constraints:

\begin{align}
\max_{\mathbf{a}} \quad & \sum_{i=1}^n U_i(s,a_i) \label{eq:objective}\\
\text{s.t.}\quad & \phip(\mathbf{a}) = 1 \label{eq:policy}\\
& \sum_{i=1}^n \risk_i(a_i) \le \tau \label{eq:risk}\\
& a_i \in \feasible_i(s),\ \forall i \in \{1,\dots,n\} \label{eq:feas}
\end{align}

\subsection{Constraint Definitions}

\textbf{Policy Feasibility Function} $\phip: \mathcal{A}_1 \times \cdots \times \mathcal{A}_n \to \{0,1\}$ encodes \emph{joint} policy validity, capturing approval chains, segregation of duties, and temporal ordering. Formally, $\phip$ evaluates a conjunction of policy predicates $\{\phi_k\}_{k=1}^K$:
\begin{equation}
\phip(\mathbf{a}) = \prod_{k=1}^K \phi_k(\mathbf{a}), \quad \phi_k: \mathcal{A}_1 \times \cdots \times \mathcal{A}_n \to \{0,1\}
\label{eq:policy_decomp}
\end{equation}

\textbf{Execution Feasibility} $\feasible_i(s) \subseteq \mathcal{A}_i$ encodes the set of actions available to agent $g_i$ given current system state $s$, incorporating RBAC permissions, resource availability, and temporal constraints:
\begin{equation}
\feasible_i(s) = \{a \in \mathcal{A}_i : \text{perm}(g_i, a, s) \wedge \text{avail}(a, s) \wedge \text{window}(a, s)\}
\label{eq:feas_def}
\end{equation}

\textbf{Risk Model.} Enterprise risk is nonlinear, state-dependent, and correlated across agents: $\risk(\mathbf{a}, s) = f\bigl(\risk_1(a_1,s), \dots, \risk_n(a_n,s)\bigr)$ where $f$ captures cross-agent interactions. For tractability, we adopt a weighted-additive decomposition:
\begin{equation}
\risk_i(a_i) = \sum_{d \in \mathcal{D}} w_d \cdot r_{i,d}(a_i)
\label{eq:risk_model}
\end{equation}
where $\mathcal{D}$ denotes risk dimensions (financial, compliance, operational, reputational) with organization-configured weights $w_d$.

\subsection{Convergence Requirements}

A coordination algorithm must satisfy: (1)~\textbf{Feasibility guarantee:} the returned $\mathbf{a}^*$ satisfies Eqs.~\eqref{eq:policy}--\eqref{eq:feas}, or the algorithm fails to safe fallback; (2)~\textbf{Bounded runtime:} termination within $K_{\max}$ iterations; (3)~\textbf{Monotonic risk reduction:} $R_{\text{tot}}^{(k+1)} \le R_{\text{tot}}^{(k)}$ when $R_{\text{tot}}^{(k)} > \tau$.

\section{CAMCO: Framework and Architecture}
\label{sec:framework}

\subsection{High-Level Architecture}

Figure~\ref{fig:arch} illustrates how \camco{} integrates within an AI-native enterprise stack via a \emph{perceive--reason--coordinate--execute--reflect} cycle.

\subsection{Core Components}

CAMCO's coordination layer comprises three components: the \textbf{Constraint Projection Engine (CPE)}, the \textbf{Risk-Weighted Utility Engine (RWUE)}, and the \textbf{Negotiation Loop}.

The \textbf{CPE} projects each agent's proposed action onto the nearest policy-compliant alternative. Given a proposed action $a_i$, the CPE solves: $a_i' = \arg\min_{a \in \feasible_i(s)} \|a - a_i\|$ subject to all active policy predicates $\{\phi_k\}$, ensuring the corrected action remains close to the agent's original intent while satisfying all hard constraints.

The \textbf{RWUE} computes adjusted utilities that internalize risk: $\tilde{U}_i = U_i(s, a_i') - \lambda \cdot \risk_i(a_i')$, where $\lambda \geq 0$ is a Lagrangian multiplier updated each round via dual-ascent: $\lambda^{(k+1)} = \lambda^{(k)} + \alpha (R_{\text{tot}}^{(k)} - \tau)^+$.

The \textbf{Negotiation Loop} orchestrates iterative re-proposals when joint actions fail risk or policy checks, tightening $\lambda$ each round until $R_{\text{tot}}$ falls below $\tau$ or $K_{\max}$ is reached. If convergence fails, a \emph{safe fallback policy} reverts all agents to their last known compliant configuration, guaranteeing no high-risk action set executes in production. This design draws on fail-safe principles from industrial control systems, where a well-defined safe state must always be reachable within bounded time. Formally, the fallback operator $\mathcal{F}$ satisfies $\Phi(\mathcal{F}(\mathbf{a})) = 1$ and $\mathcal{R}(\mathcal{F}(\mathbf{a}), s) \leq \tau$ for all states $s$, providing a constructive proof that the orchestration never deadlocks in a non-compliant configuration.

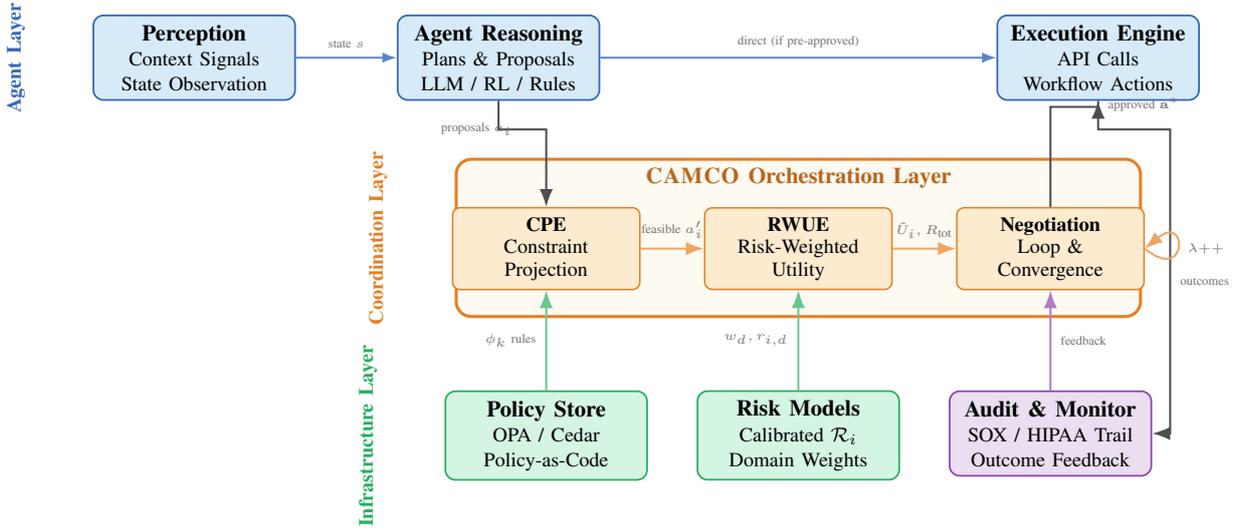
\begin{figure*}[!t]
\centering
\resizebox{0.9\textwidth}{!}{%
\begin{tikzpicture}[
  node distance=10mm and 12mm,
  agentbox/.style={draw, rounded corners=4pt, align=center, minimum width=2.8cm, minimum height=1.1cm, font=\small, fill=lightblue, draw=agentblue, line width=0.7pt},
  camcoinner/.style={draw, rounded corners=3pt, align=center, minimum width=2.6cm, minimum height=0.9cm, font=\footnotesize, fill=lightorange, draw=camcoorange, line width=0.6pt},
  infrabox/.style={draw, rounded corners=4pt, align=center, minimum width=2.8cm, minimum height=1.1cm, font=\small, fill=lightgreen, draw=infragreen, line width=0.7pt},
  auditbox/.style={draw, rounded corners=4pt, align=center, minimum width=2.8cm, minimum height=1.1cm, font=\small, fill=lightpurple, draw=auditpurple, line width=0.7pt},
  layerlabel/.style={font=\footnotesize\bfseries, rotate=90, anchor=south},
  arrow/.style={-{Latex[length=2.5mm]}, thick, color=black!70},
  dataarrow/.style={-{Latex[length=2mm]}, thick, color=agentblue!80},
  note/.style={font=\tiny, color=black!55}
]

\node[agentbox] (percept) {\textbf{Perception}\\\footnotesize Context Signals\\\footnotesize State Observation};
\node[agentbox, right=14mm of percept] (reason) {\textbf{Agent Reasoning}\\\footnotesize Plans \& Proposals\\\footnotesize LLM / RL / Rules};
\node[agentbox, right=55mm of reason] (exec) {\textbf{Execution Engine}\\\footnotesize API Calls\\\footnotesize Workflow Actions};

\node[draw, rounded corners=6pt, fill=lightorange!30, draw=camcoorange, line width=1.2pt,
      minimum width=9.5cm, minimum height=2.2cm, below=14mm of $(reason)!0.5!(exec)$] (camcobg) {};
\node[font=\small\bfseries, color=camcoorange!80!black, anchor=north] at (camcobg.north) {\camco{} Orchestration Layer};

\node[camcoinner] (cpe) at ($(camcobg.center)+(-3.5,-.15)$) {\textbf{CPE}\\Constraint\\Projection};
\node[camcoinner] (rwue) at ($(camcobg.center)+(0,-.15)$) {\textbf{RWUE}\\Risk-Weighted\\Utility};
\node[camcoinner] (negot) at ($(camcobg.center)+(3.5,-.15)$) {\textbf{Negotiation}\\Loop \&\\Convergence};

\draw[arrow, camcoorange!70] (cpe) -- (rwue) node[note, above, midway] {feasible $a_i'$};
\draw[arrow, camcoorange!70] (rwue) -- (negot) node[note, above, midway] {$\tilde{U}_i, R_{\text{tot}}$};

\node[infrabox, below=14mm of cpe] (policy) {\textbf{Policy Store}\\\footnotesize OPA / Cedar\\\footnotesize Policy-as-Code};
\node[infrabox, below=14mm of rwue] (riskdb) {\textbf{Risk Models}\\\footnotesize Calibrated $\risk_i$\\\footnotesize Domain Weights};
\node[auditbox, below=14mm of negot] (audit) {\textbf{Audit \& Monitor}\\\footnotesize SOX / HIPAA Trail\\\footnotesize Outcome Feedback};

\node[layerlabel, color=agentblue] at ($(percept.west)+(-8mm,0)$) {Agent Layer};
\node[layerlabel, color=camcoorange] at ($(camcobg.west)+(-8mm,0)$) {Coordination Layer};
\node[layerlabel, color=infragreen] at ($(policy.west)+(-8mm,0)$) {Infrastructure Layer};

\draw[dataarrow] (percept) -- (reason) node[note, above, midway] {state $s$};
\draw[dataarrow] (reason) -- (exec) node[note, above, midway] {direct (if pre-approved)};
\draw[arrow] (reason.south) -- ++(0,-4mm) -| (cpe.north) node[note, left, pos=0.25] {proposals $a_i$};
\draw[arrow] (negot.north) -- ++(0,14mm) -| (exec.south) node[note, pos=0.75, right] {approved $\mathbf{a}^*$};
\draw[arrow, infragreen!70] (policy) -- (cpe) node[note, left, midway] {$\phi_k$ rules};
\draw[arrow, infragreen!70] (riskdb) -- (rwue) node[note, left, midway] {$w_d, r_{i,d}$};
\draw[arrow, auditpurple!70] (audit) -- (negot) node[note, right, midway] {feedback};
\draw[arrow] (exec.south) -- ++(0,-4mm) -- ++(10mm,0) |- (audit.east) node[note, pos=0.25, right] {outcomes};
\draw[arrow, camcoorange!70] (negot.east) .. controls ++(6mm,-5mm) and ++(6mm,5mm) .. (negot.east) node[note, right, pos=0.5] {$\lambda{+}{+}$};

\end{tikzpicture}%
}
\caption{\camco{} architecture within an AI-native enterprise stack. The \textbf{Agent Layer} (blue) handles perception and reasoning; the \textbf{Coordination Layer} (orange) contains CAMCO's three core components; the \textbf{Infrastructure Layer} (green/purple) provides policy stores, risk models, and audit services.}
\label{fig:arch}
\end{figure*}

\FloatBarrier

\textbf{(1) Constraint Projection Engine (CPE).} The CPE transforms each agent proposal $a_i$ into a policy-feasible action $a_i' = \Pi_{\mathcal{C}}(a_i)$ by solving a nearest-feasible-point problem over $\mathcal{C}_i(s) = \feasible_i(s) \cap \{a : \phi_k(\cdot) = 1, \forall k\}$. For convex continuous action spaces, $\Pi_{\mathcal{C}}$ is the standard Euclidean projection; for discrete enterprise actions, it is a constraint-satisfaction search with minimum edit distance. The projection satisfies idempotency ($\Pi_{\mathcal{C}}(\Pi_{\mathcal{C}}(a)) = \Pi_{\mathcal{C}}(a)$), guaranteed feasibility, and proximity (nearest feasible point). Figure~\ref{fig:projection} illustrates the projection process.

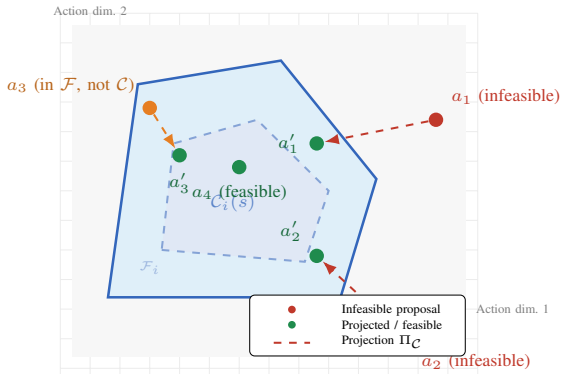
\begin{figure}[!b]
\centering
\resizebox{0.85\columnwidth}{!}{%
\begin{tikzpicture}[scale=0.82]
  \draw[gray!15, very thin, step=0.5] (-0.5,-1.2) grid (6.5,5);
  
  \draw[-{Latex[length=2mm]}, thick, black!50] (-0.3,0) -- (6.3,0) node[right, font=\tiny] {Action dim.\ 1};
  \draw[-{Latex[length=2mm]}, thick, black!50] (0,-0.8) -- (0,4.8) node[above, font=\tiny] {Action dim.\ 2};
  
  \fill[black!3] (-0.3,-0.8) rectangle (6.3,4.8);
  \fill[feasibleblue!40] (0.3,0.2) -- (4.2,0.2) -- (4.8,2.2) -- (3.2,4.2) -- (0.8,3.8) -- cycle;
  \draw[agentblue, thick, line width=1pt] (0.3,0.2) -- (4.2,0.2) -- (4.8,2.2) -- (3.2,4.2) -- (0.8,3.8) -- cycle;
  \fill[agentblue!15] (1.2,1.0) -- (3.6,0.8) -- (4.0,2.0) -- (2.8,3.2) -- (1.4,2.8) -- cycle;
  \draw[agentblue!60, dashed, thick] (1.2,1.0) -- (3.6,0.8) -- (4.0,2.0) -- (2.8,3.2) -- (1.4,2.8) -- cycle;
  
  \node[font=\scriptsize, agentblue!80] at (2.4,1.8) {$\mathcal{C}_i(s)$};
  \node[font=\tiny, agentblue!50] at (1.0,0.7) {$\feasible_i$};
  
  \node[circle, fill=policyred, inner sep=2pt, label={[font=\scriptsize, policyred]above right:$a_1$ (infeasible)}] (a1) at (5.8, 3.2) {};
  \node[circle, fill=policyred, inner sep=2pt, label={[font=\scriptsize, policyred]below right:$a_2$ (infeasible)}] (a2) at (5.3, -0.5) {};
  \node[circle, fill=camcoorange, inner sep=2pt, label={[font=\scriptsize, camcoorange!80!black]above left:$a_3$ (in $\feasible$, not $\mathcal{C}$)}] (a3) at (1.0, 3.4) {};
  
  \node[circle, fill=infragreen!80!black, inner sep=2pt, label={[font=\scriptsize, infragreen!60!black]left:$a_1'$}] (p1) at (3.8, 2.8) {};
  \node[circle, fill=infragreen!80!black, inner sep=2pt, label={[font=\scriptsize, infragreen!60!black]above left:$a_2'$}] (p2) at (3.8, 0.9) {};
  \node[circle, fill=infragreen!80!black, inner sep=2pt, label={[font=\scriptsize, infragreen!60!black]below:$a_3'$}] (p3) at (1.5, 2.6) {};
  \node[circle, fill=infragreen!80!black, inner sep=2pt, label={[font=\scriptsize, infragreen!60!black]below:$a_4$ (feasible)}] (a4) at (2.5, 2.4) {};
  
  \draw[-{Latex[length=2mm]}, dashed, policyred, thick] (a1) -- (p1);
  \draw[-{Latex[length=2mm]}, dashed, policyred, thick] (a2) -- (p2);
  \draw[-{Latex[length=2mm]}, dashed, camcoorange, thick] (a3) -- (p3);
  
  \node[draw, fill=white, rounded corners=2pt, inner sep=3pt, font=\tiny, anchor=south east] at (6.3,-0.8) {
    \begin{tabular}{cl}
    \tikz\fill[policyred] (0,0) circle (1.5pt); & Infeasible proposal \\
    \tikz\fill[infragreen!80!black] (0,0) circle (1.5pt); & Projected / feasible \\
    \tikz\draw[policyred, dashed, thick] (0,0) -- (5mm,0); & Projection $\Pi_{\mathcal{C}}$
    \end{tabular}
  };
\end{tikzpicture}%
}
\caption{Constraint projection visualization. Infeasible proposals (red) are projected to the nearest feasible point (green) within the policy-compliant subset $\mathcal{C}_i(s)$.}
\label{fig:projection}
\end{figure}

\textbf{(2) Risk-Weighted Utility Engine (RWUE).} The RWUE dynamically adjusts each agent's utility with a Lagrangian risk penalty:
\begin{equation}
\tilde{U}_i(s,a_i; \lambda) = U_i(s,a_i) - \lambda \cdot \risk_i(a_i)
\label{eq:shaped}
\end{equation}
The multiplier $\lambda$ increases monotonically when aggregate risk exceeds $\tau$, ensuring monotonic risk reduction. Concretely, the update rule $\lambda^{(t+1)} = \lambda^{(t)} + \eta \cdot \max(0, R_{\text{tot}}^{(t)} - \tau)$ follows a subgradient ascent on the dual variable, where the step size $\eta$ controls the trade-off between convergence speed and oscillation. In practice, CAMCO uses a diminishing step size $\eta_t = \eta_0 / \sqrt{t}$ to guarantee convergence to the saddle point of the Lagrangian while permitting aggressive early tightening. This approach is closely related to classical Lagrangian relaxation in constrained optimization~\cite{Boyd2004Convex}, adapted here for the multi-agent, multi-constraint setting where each agent's risk contribution is aggregated before the dual update.

The risk model $\risk_i(a_i)$ combines domain-specific indicators whose weights are configured per deployment, decoupling risk semantics from optimization. In practice, convergence occurs within 2--4 rounds; $K_{\max}$ is reached in only $\sim$3\% of proposals.

Lagrangian dual decomposition was chosen for: (i)~\emph{privacy}---agents report only projected actions~\cite{Shoham2009MAS}; (ii)~\emph{continuous trade-off}---quantitative utility--risk balance; (iii)~\emph{auditability}---$\lambda$ acts as a shadow price of risk, transparent to regulators. The pipeline terminates in at most $K_{\max}$ rounds (Theorem~1): monotonic increase of $\lambda$ shrinks each agent's rational action set until convergence or fallback.

\textbf{(3) Negotiation Loop.} The negotiation loop implements an iterative propose--project--evaluate cycle: each agent proposes an action maximizing $\tilde{U}_i$, the CPE projects it, aggregate risk is evaluated, and if $\phip(\mathbf{a}')=1$ and $R_{\text{tot}} \le \tau$ the joint action is accepted; otherwise $\lambda$ increases and agents re-propose. Figure~\ref{fig:negotiation} visualizes this protocol.

\begin{figure}[!t]
\centering
\resizebox{0.85\columnwidth}{!}{%
\begin{tikzpicture}[
  node distance=7mm,
  procstep/.style={draw, rounded corners=3pt, align=center, minimum width=2.8cm, minimum height=0.75cm, font=\scriptsize, fill=lightblue, draw=agentblue, line width=0.6pt},
  camcostep/.style={draw, rounded corners=3pt, align=center, minimum width=2.8cm, minimum height=0.75cm, font=\scriptsize, fill=lightorange, draw=camcoorange, line width=0.6pt},
  decision/.style={draw, diamond, aspect=1.8, align=center, minimum width=2cm, minimum height=1cm, font=\scriptsize, fill=yellow!15, draw=camcoorange, line width=0.6pt},
  accept/.style={draw, rounded corners=3pt, align=center, minimum width=2.8cm, minimum height=0.75cm, font=\scriptsize, fill=lightgreen, draw=infragreen, line width=0.6pt},
  fail/.style={draw, rounded corners=3pt, align=center, minimum width=2.0cm, minimum height=0.65cm, font=\scriptsize, fill=policyred!15, draw=policyred, line width=0.6pt},
  arrow/.style={-{Latex[length=2mm]}, thick, color=black!70},
  note/.style={font=\tiny, color=black!50}
]
\node[font=\scriptsize\bfseries, color=camcoorange!80!black] (iter) {Iteration $k = 1, \dots, K_{\max}$};
\node[procstep, below=5mm of iter] (propose) {\textbf{1. Propose} (parallel)\\$a_i^{(k)} = \arg\max \tilde{U}_i(s,\cdot;\lambda)$};
\node[camcostep, below=of propose] (project) {\textbf{2. Constraint Projection}\\$a_i' = \Pi_{\mathcal{C}}(a_i^{(k)})$};
\node[camcostep, below=of project] (risk) {\textbf{3. Risk Evaluation}\\$R_{\text{tot}} = \sum_i \risk_i(a_i')$};
\node[decision, below=8mm of risk] (check) {$\Phi(\mathbf{a}')=1$\\$\wedge\ R_{\text{tot}} \le \tau$\,?};
\node[accept, below=8mm of check] (accept) {\textbf{Accept} $\mathbf{a}^*$\\Execute \& Log};
\node[procstep, right=14mm of check] (update) {\textbf{Update}\\$\lambda \gets \lambda + \delta \cdot \frac{R}{\tau}$};
\node[fail, left=14mm of check] (failnode) {\textbf{Fail}\\Human Review};

\draw[arrow] (propose) -- (project) node[note, right, midway] {$n$ proposals};
\draw[arrow] (project) -- (risk) node[note, right, midway] {feasible $a_i'$};
\draw[arrow] (risk) -- (check) node[note, right, midway] {$R_{\text{tot}}, \Phi$};
\draw[arrow] (check) -- node[note, left] {Yes} (accept);
\draw[arrow] (check) -- node[note, above] {No, $k < K_{\max}$} (update);
\draw[arrow] (check) -- node[note, above] {No, $k = K_{\max}$} (failnode);
\draw[arrow] (update) |- (propose) node[note, right, pos=0.25] {re-negotiate};
\end{tikzpicture}%
}
\caption{\camco{} negotiation protocol. Agents propose in parallel, proposals are projected and risk-evaluated. If unsatisfied, $\lambda$ increases and agents re-propose; after $K_{\max}$ failures, escalation to human review.}
\label{fig:negotiation}
\end{figure}
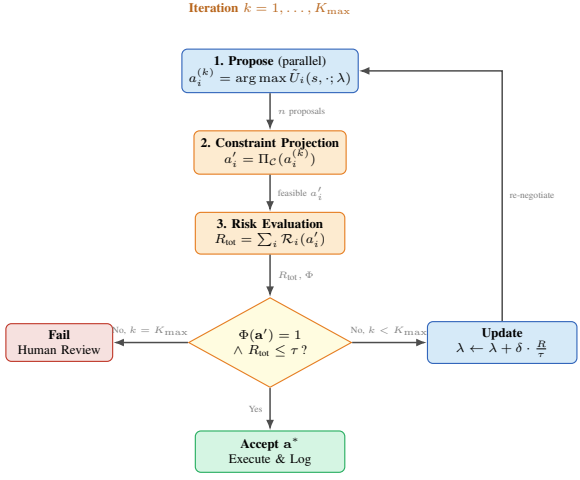

\section{CAMCO Algorithm}
\label{sec:algorithm}

\subsection{Formal Algorithm}

Algorithm~\ref{alg:camco} presents the complete \camco{} coordination procedure.

\begin{algorithm}[t]
\caption{\camco{}: Constraint-Aware Multi-Agent Cognitive Orchestration}
\label{alg:camco}
\begin{algorithmic}
\Require Agents $G=\{g_i\}_{i=1}^n$, state $s$, risk threshold $\tau$, step size $\delta>0$, initial multiplier $\lambda_0 \geq 0$, max iterations $K_{\max}$
\Ensure Feasible joint action $\mathbf{a}^*$ or \texttt{fail}
\State $\lambda \gets \lambda_0$
\State $R_{\text{prev}} \gets \infty$
\For{$k = 1$ to $K_{\max}$}
  \ForAll{agents $g_i \in G$ \textbf{in parallel}}
    \State $a_i \gets \arg\max_{a \in \mathcal{A}_i} \tilde{U}_i(s, a; \lambda)$
    \State $a_i \gets \Pi_{\mathcal{C}}(a_i)$
    \If{$a_i = \texttt{reject}$}
      \State $a_i \gets \textsc{SafeDefault}(g_i, s)$
    \EndIf
    \State $r_i \gets \risk_i(a_i)$
  \EndFor
  \State $R_{\text{tot}} \gets \sum_{i=1}^n r_i$
  \If{$\phip(a_1,\dots,a_n)=1$ \textbf{and} $R_{\text{tot}} \le \tau$}
    \State \textbf{return} $\mathbf{a}^*=(a_1,\dots,a_n)$
  \Else
    \State $\lambda \gets \lambda + \delta \cdot \max(1, R_{\text{tot}}/\tau)$
    \State $R_{\text{prev}} \gets R_{\text{tot}}$
  \EndIf
\EndFor
\State \textbf{return} \texttt{fail}
\end{algorithmic}
\end{algorithm}

\subsection{Convergence Analysis}

\begin{theorem}[Bounded Convergence]
\label{thm:convergence}
If each agent's risk-shaped utility $\tilde{U}_i(s, \cdot; \lambda)$ is continuous in $\lambda$ and each $\risk_i$ is bounded above by $\bar{r}_i < \infty$, then \camco{} either finds a feasible joint action or returns \texttt{fail} within $K_{\max}$ iterations.
\end{theorem}

\begin{proof}[Proof sketch.]
At each iteration where $R_{\text{tot}} > \tau$, $\lambda$ increases by at least $\delta > 0$. As $\lambda \to \infty$, the risk-shaped utility makes any positive-risk action arbitrarily undesirable. Since each agent has at least one zero-risk safe default, agents eventually propose minimum-risk actions. If $\sum_i \min_{a_i \in \feasible_i(s)} \risk_i(a_i) \le \tau$, convergence is guaranteed; otherwise \texttt{fail} is returned. The process terminates at $K_{\max}$ regardless. \qed
\end{proof}

\begin{proposition}[Monotonic Risk Reduction]
\label{prop:monotone}
Under \camco{}, if agents respond monotonically to $\lambda$ (higher $\lambda$ leads to weakly lower-risk proposals), then $R_{\text{tot}}^{(k+1)} \le R_{\text{tot}}^{(k)}$ for all $k$ where $R_{\text{tot}}^{(k)} > \tau$.
\end{proposition}

\section{Evaluation}
\label{sec:evaluation}

\subsection{Simulation Scenarios}

We evaluate \camco{} using three enterprise-inspired simulation scenarios, each instantiating the constrained decision model from Section~\ref{sec:formulation}:

\begin{itemize}
\item \textbf{S1: Financial Approval Workflow.} $n=4$ agents (requester, manager, compliance officer, CFO) coordinating multi-stage payment approvals with escalation rules, amount-based thresholds, and segregation-of-duties constraints. Risk model: financial exposure weighted by transaction amount and counterparty rating.

\item \textbf{S2: Payroll Adjustment Task.} $n=3$ agents (HR analyst, payroll processor, audit reviewer) handling salary adjustments with compliance checks, executive compensation sensitivity thresholds, and retroactive adjustment limits. Risk model: compliance violation probability and audit exposure.

\item \textbf{S3: Hybrid Cloud Deployment.} $n=5$ agents (developer, DevOps, security reviewer, change manager, SRE) managing production deployments with RBAC-gated permissions, data sovereignty constraints, change window restrictions, and rollback readiness. Risk model: service disruption probability and RTO deviation.
\end{itemize}

Each scenario is run for 500 episodes with randomized initial states.

\subsection{Baselines}

\begin{enumerate}
\item \textbf{B1: Unconstrained MARL} -- agents maximize individual reward without constraint enforcement (reward-first).
\item \textbf{B2: Centralized greedy optimizer} -- single planner selects highest-utility joint action without risk bounding (utility-first heuristic).
\item \textbf{B3: Static rule/workflow engine} -- evaluates policy constraints only, rejects non-compliant actions without optimization or negotiation (policy-only).
\item \textbf{B4: Lagrangian MARL} -- each agent independently applies Lagrangian constraint handling (per-agent $\lambda$ update) without joint constraint projection or multi-agent negotiation protocol (constrained-but-uncoordinated).
\end{enumerate}

\subsection{Metrics}

Table~\ref{tab:metrics} defines evaluation metrics aligned to enterprise concerns.

\begin{table}[t]
\caption{Evaluation Metrics}
\label{tab:metrics}
\centering
\scriptsize
\begin{tabular}{p{1.8cm}p{5.4cm}}
\toprule
Metric & Definition \\
\midrule
Violation Rate & Fraction of joint actions violating policy feasibility ($\phip=0$). \\
Risk Exposure & Mean aggregate risk $\bar{R}_{\text{tot}}=\mathbb{E}[\sum_i \risk_i(a_i)]$ relative to $\tau$. \\
Deadlock Rate & Fraction of episodes failing to reach feasible joint action within $K_{\max}$. \\
Convergence Its. & Mean iterations until acceptance (lower is better). \\
Utility Retention & $\frac{\sum_i U_i(\text{accepted})}{\sum_i U_i(\text{unconstrained best})} \times 100\%$. \\
\bottomrule
\end{tabular}
\end{table}

\subsection{Results}

Table~\ref{tab:results} presents results across all scenarios and methods.

\begin{table}[t]
\caption{Evaluation Results Across Enterprise Scenarios}
\label{tab:results}
\centering
\scriptsize
\begin{tabular}{llccccc}
\toprule
Scen. & Method & Viol.\% & Risk$^\dagger$ & Dead.\% & Conv. & Util.\% \\
\midrule
\multirow{5}{*}{S1} & B1 (MARL) & 23.4 & 1.42 & 0.0 & 1.0 & 100.0 \\
& B2 (Greedy) & 8.2 & 1.31 & 0.0 & 1.0 & 98.7 \\
& B3 (Rules) & 0.0 & 0.65 & 12.8 & 1.0 & 71.3 \\
& B4 (Lag.\ MARL) & 4.6 & 0.91 & 3.2 & 1.0 & 93.4 \\
& \textbf{\camco{}} & \textbf{0.0} & \textbf{0.68} & \textbf{1.2} & \textbf{2.1} & \textbf{96.8} \\
\midrule
\multirow{5}{*}{S2} & B1 (MARL) & 18.6 & 1.28 & 0.0 & 1.0 & 100.0 \\
& B2 (Greedy) & 5.4 & 1.19 & 0.0 & 1.0 & 99.1 \\
& B3 (Rules) & 0.0 & 0.58 & 8.4 & 1.0 & 74.6 \\
& B4 (Lag.\ MARL) & 3.8 & 0.88 & 2.6 & 1.0 & 91.7 \\
& \textbf{\camco{}} & \textbf{0.0} & \textbf{0.72} & \textbf{0.8} & \textbf{2.4} & \textbf{94.2} \\
\midrule
\multirow{5}{*}{S3} & B1 (MARL) & 31.2 & 1.67 & 0.0 & 1.0 & 100.0 \\
& B2 (Greedy) & 12.8 & 1.48 & 0.0 & 1.0 & 97.3 \\
& B3 (Rules) & 0.0 & 0.52 & 18.6 & 1.0 & 65.8 \\
& B4 (Lag.\ MARL) & 9.4 & 1.08 & 5.8 & 1.0 & 89.3 \\
& \textbf{\camco{}} & \textbf{0.0} & \textbf{0.74} & \textbf{2.0} & \textbf{2.8} & \textbf{92.1} \\
\bottomrule
\multicolumn{7}{l}{\tiny $^\dagger$Risk expressed as ratio $\bar{R}_{\text{tot}}/\tau$; values $>1.0$ indicate threshold violations.}
\end{tabular}
\end{table}

\textbf{Key findings.}
\camco{} achieves \textbf{zero policy violations} across all scenarios, outperforming B1 (18--31\%), B2 (5--13\%), and B4 (4--9\%). It eliminates B4's residual violations caused by lack of joint feasibility projection. Risk exposure stays well below threshold ($\bar{R}/\tau \in [0.68, 0.74]$), utility retention reaches \textbf{92--97\%} (vs.\ B3's 66--75\%), deadlock rates remain \textbf{below 2\%} (vs.\ B3's 8--19\%), and mean convergence occurs in \textbf{2.1--2.8 iterations}. Sensitivity analysis confirms 100\% compliance across $\tau \in [0.4, 1.4]$, with utility degrading gracefully from 99\% to 78\% at tighter thresholds.

\subsection{Limitations}

\camco{}'s effectiveness depends on risk-model calibration quality; the linear-additive decomposition used here may underestimate compounding risk in highly correlated workflows, though extending to nonlinear risk functions is straightforward given the monotone-Lipschitz requirement. The projection operator's efficiency relies on convexity for continuous spaces and moderate $|\mathcal{A}_i|$ for discrete; scaling beyond $|\mathcal{A}_i| > 10^3$ would require constraint-programming solvers.

\section{Conclusion}

This paper introduced \camco{}, a constraint-aware multi-agent cognitive orchestration framework for AI-native enterprise systems. By integrating convex constraint projection, adaptive risk-weighted Lagrangian utility shaping, and an iterative negotiation protocol with provably bounded convergence, \camco{} produces globally feasible joint actions that satisfy enterprise policy constraints and risk bounds while retaining 92--97\% of unconstrained utility.

Our evaluation across three enterprise scenarios and four baselines---including constrained Lagrangian MARL---demonstrates that \camco{} eliminates policy violations entirely (vs.\ 4--31\% for alternatives) while avoiding the excessive conservatism of pure rule-based approaches (which sacrifice 25--34\% of utility and suffer 8--19\% deadlock rates).

\textbf{Future work} includes adaptive risk calibration via Bayesian updating, hierarchical coordination for large-scale deployments ($n > 20$), and empirical validation with production policy engines (OPA, Cedar).


\end{document}